\documentclass{article}

\usepackage[preprint]{neurips_2023}

\usepackage{amsmath,amsfonts,bm}

\def\eqref#1{equation~\ref{#1}}

\def\1{\bm{1}}

\DeclareMathAlphabet{\mathsfit}{\encodingdefault}{\sfdefault}{m}{sl}
\SetMathAlphabet{\mathsfit}{bold}{\encodingdefault}{\sfdefault}{bx}{n}

\usepackage{hyperref}
\hypersetup{
    colorlinks=true,
    linkcolor=red,
    filecolor=magenta,
    urlcolor=blue,
    citecolor=purple,
    pdftitle={Overleaf Example},
    pdfpagemode=FullScreen,
    }

\usepackage[T1]{fontenc}    %

\usepackage{url}
\usepackage{amsfonts}       %
\usepackage{amsmath,amssymb}
\usepackage{amsthm}
\usepackage{nicefrac}       %
\usepackage{microtype}      %
\usepackage{url}
\usepackage{array,booktabs,makecell}
\usepackage{hyperref}
\usepackage[capitalize]{cleveref}
\usepackage{autonum}
\usepackage{mathtools}
\usepackage[colorinlistoftodos]{todonotes}
\usepackage{enumitem}
\usepackage[parfill]{parskip}
\usepackage{bm}
\usepackage{mathrsfs}
\usepackage{subfigure} 
\usepackage{dblfloatfix}
\usepackage{multirow}
\usepackage{wrapfig}
\usepackage{multicol}
\usepackage{caption}

\usepackage{siunitx}
\usepackage{dsfont}

\usepackage[scr=esstix]{mathalfa} %
\usepackage{mathtools}
\usepackage{float}
\graphicspath{{figures/}}

\usepackage{algorithmic}
\usepackage{algorithm}
\usepackage{changepage}

\usepackage{booktabs} 
\usepackage[flushleft]{threeparttable}
\usepackage{mathdots}
\usepackage{blkarray}

\usepackage{tikz}
\usetikzlibrary{arrows,shapes,snakes,automata,backgrounds,petri}

\usepackage{titletoc}

\newtheorem{theorem}{Theorem}

\newtheorem{lemma}{Lemma}

\newtheorem{definition}{Definition}
\newtheorem{assumption}{Assumption}

\newcommand{\closer}[3]{{\kern-#1ex{#2}\kern-#3ex}}

\DeclarePairedDelimiterX{\infdivx}[2]{(}{)}{%
  #1\;\delimsize\|\;#2%
}

\usepackage{tikz}

\newcounter{commentCounter}
\newif\iftrvar
\trvarfalse
\iftrvar
\newcommand{\am}[1]{{\small \color{blue} \refstepcounter{commentCounter}\textsf{[AM]$_{\arabic{commentCounter}}$:{#1}}}}

\else
\newcommand{\am}[1]{}

\fi

\definecolor{lightgray}{rgb}{.9,.9,.9}
\definecolor{darkgray}{rgb}{.4,.4,.4}
\definecolor{purple}{rgb}{0.65, 0.12, 0.82}
\definecolor{darkgreen}{rgb}{0, 0.365, 0}
\definecolor{orange}{rgb}{1,0.5,0}
\definecolor{deep}{rgb}{0.13, 0.46, 0.7}

\title{Generalization Across Observation Shifts in Reinforcement Learning}

\author{Anuj Mahajan  \thanks{Corresponding author. Part of work done as PhD student at University of Oxford.} \\
Amazon\\
\texttt{anuj.mahajan.ai@gmail.com} \\
\And %
Amy Zhang \\
University of Texas at Austin \\
\texttt{amy.zhang@austin.utexas.edu} \\
}

\begin{document}

\maketitle

\begin{abstract}
Learning policies which are robust to changes in the environment are critical for real world deployment of Reinforcement Learning agents. They are also necessary for achieving good generalization across environment shifts.
We focus on bisimulation metrics, which provide a powerful means for abstracting task relevant components of the observation and learning a succinct representation space for training the agent using reinforcement learning. In this work, we extend the bisimulation framework to also account for context dependent observation shifts. Specifically, we focus on the simulator based learning setting and use alternate observations to learn a representation space which is invariant to observation shifts using a novel bisimulation based objective. This allows us to deploy the agent to varying observation settings during test time and generalize to unseen scenarios. We further provide novel theoretical bounds for simulator fidelity and performance transfer guarantees for using a learnt policy to unseen shifts. Empirical analysis on the high-dimensional image based control domains demonstrates the efficacy of our method.
\end{abstract}

\vspace{-3mm}
\section{Introduction}
\label{ch8_intro}
\vspace{-2mm}
Many practical scenarios in reinforcement learning (RL) applications require the agent to be robust to changes in the observations space between training and deployment. Such changes can occur due to lack of complete information about the deployment environment which often happens as the training environment is usually highly controlled or simulators are used for training the agent, both of these scenarios seldom capture the complexity and noisiness of the real world. Moreover, these changes can also occur due to various practical errors and constraints under which autonomous agents need to be deployed (e.g. variations in sensor position and fitting on automobiles, change in calibration settings of visual input, change in sensor types due to upgrades, calibration changes due to wear and tear etc.).
\vspace{-2mm}

While existing methods aimed at obtaining better generalization in RL can be partially applied to the above problem, they hardly utilise the rich underlying structure that can enable efficient learning of policies which generalize well across the observation shifts. For instance, methods like domain randomization \citep{zhao2020sim} which work well for supervised perception problems in robotics are insufficient for obtaining good performance on control tasks. Similarly, methods for RL which aim at using unsupervised data for learning control representations: data augmentation \citep{laskin2020reinforcement, kostrikov2020image}, contrastive learning \citep{oord2018representation}, reconstruction \citep{Lange2012Autonomous, hafner2018planet} are not well aligned with the objective of maximizing rewards in complex domains. Further, the presence of task irrelevant noise in the environment make it difficult for these methods to generalise across the changes in the observation space. While, state abstraction based methods like bisimulation \citep{zhangLearningInvariantRepresentations2021, gelada2019deepmdp} to which our method is closely related, can help ignore irrelevant task features, they do not fully exploit the structure present in observation shift setting towards ensuring better generalization.

\vspace{-2mm}
In this work, we propose a novel solution to the aforementioned problem using the concept of conditional bisimulation and application of simulator/specialized setup during train time which help explicitly teach the agent, the similarities across changes in the observation space. Our method leverages the MDP level isomorphism \citep{ravindran2004algebraic} in the observation shift setting for obtaining a richer representation loss. Our methods offers two-fold advantage: 
\begin{itemize}
    \item We can learn representations which are robust to shifts in observation space in a sample efficient manner.
    \item We learn to ignore task irrelavant features as our metric is grounded in rewards.
\end{itemize}
\vspace{-2mm}
\section{Background}
\label{background}\vspace{-2mm}
\textbf{Reinforcement Learning: }
A Markov Decision Process (MDP) is formally defined as a tuple $ \left\langle S,U,P,r,\gamma,\rho\right\rangle$. 
Here $S$ is the state space of the environment and $\rho$ is the initial state distribution. At each time step $t$, an agent observes the state $s \in S$ and chooses an action $a \in U$ using its policy $\pi: S \rightarrow \mathcal{P}(U)$, where $\mathcal{P}(\cdot)$ represents the space of distributions on the argument set.
This leads to a state transition governed by the distribution $P(s'|s,a):S\times{U}\times S\rightarrow [0,1]$, and the agent receives reward $r(s,{a}):S \times {U} \rightarrow [0,R_{max}]$ which can be potentially stochastic. We consider the discounted infinite horizon setting, where the discount factor is given by $\gamma \in [0,1)$. 
The state-action trajectory of the agent is represented by $\tau\in T\equiv(S\times U)^*$, we overload the notation to also include rewards as necessary. 
The value of a policy is defined as: $J^\pi =  \mathbb{E}_{\pi,\rho} \left[  \sum_{t=0}^\infty \gamma^t r_{\tau}(s_t) \right]$
where the expectation on the RHS is well defined given bounds on rewards and $\gamma$. We also define three other useful functions: \smash{(1) $Q^{\pi}(s, a) = \mathbb{E}_{\pi}[\sum_{t=0}^{\infty}\gamma^t r(s_t) \vert s_0 = s, a_0=a]$}, (2) $V^\pi(s) =\mathbb{E}_{a\sim\pi}Q^\pi(s,a)$, (3) $A^\pi(s, a) = Q^\pi(s,a) - V^\pi(s)$,
respectively called the action-value, value and advantage functions. The goal of the MDP problem is to find the optimal policy $\pi^{*}$ corresponding to the optimal policy value $J^*$. It is well known that a deterministic optimal policy always exists for finite MDPs. Further, the optimal value function $V^*$ and optimal action value function $Q^*$ also exhibit important properties like uniqueness and point-wise function dominance over the entire domain \citep{sutton2011reinforcement}. A standard assumption in the reinforcement learning (RL) setting is that both the rewards and transition kernels are not known to the agent. \vspace{1mm}
\label{ch2cmdp}
\begin{wrapfigure}{r}{0.35\linewidth}
\centering
\vspace{-10pt}
\includegraphics[width=0.9\linewidth]{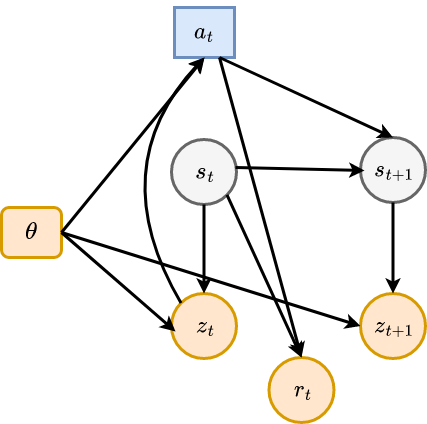}
\caption{Parametrized observation setting.
\label{fig:pgm_setting}}
\vspace{-15pt}
\end{wrapfigure}
\textbf{Rich observations and Contextual MDPs: \label{rocmdp}} An important class of MDP arises when we consider the presence of an underlying parametrized context $\theta$, which governs the rewards and transitions in the MDP framework. This extension is called the Contextual MDP setting (CMDP, \cite{hallak2015contextual}). Formally, we have $\mathcal{M}\triangleq \left\langle S,  U,  P_{\theta},  r_{\theta}, \gamma, \rho, \Theta,  P_\Theta\right\rangle$, where $\Theta$ defines a space of context parameters, $ P_\Theta$ is a fixed distribution over the contexts. Thus the transitions $P_\theta:S\times{U}\times S\times\Theta\rightarrow [0,1]$, and the agent reward $r_\theta:S \times {U}\times\Theta \rightarrow [0,R_{max}]$ are now also functions of the context parameter $\theta$. 
We now discuss the CMDP setting used in this work: we consider a parametrized context which defines a functional transformation of the underlying MDP state giving rise to context dependent observations. Formally, we have $\mathcal{M}\triangleq \left\langle S,  U,  P,  r, \gamma, \rho, \Theta,  P_\Theta,  Z, f\right\rangle$, where $\Theta$ defines a space of context parameters, $ P_\Theta$ is a fixed distribution over the contexts, $ Z$ is the set of observations emitted as $f:  S \times \Theta \to  Z$. Thus fixing a particular context $\theta$ gives us a richly observed MDP \citep{krishnamurthy2016pac, mahajan2023a} indexed by $\theta$: $\mathcal{M}_\theta$. We assume that the agent observes $\theta$ in our setting. \cref{fig:pgm_setting} illustrates the parametrized observation setting. Without loss of generality, we assume $ S\subset [0,1]^n$, $ Z\subset [0,1]^l$, where typically $n<<l$. We will use $f(s, \theta), f_\theta(s)$ interchangeably to highlight the corresponding (un)-curried versions of the observation function. We will be focusing on functional forms for observations, but the setting can be extended to scenarios with added independent or correlated noise at each step with suitable assumptions about identifiability~\citep{zhangInvariantCausalPrediction2020}.

\vspace{-2mm}\textbf{Bisimulation: }
MDP Bisimulation defines a notion of state abstraction which groups states that are behaviorally equivalent~\citep{li2006stateabs}.
Two states $s_i$ and $s_j$are bisimilar if they both share the same immediate reward and equivalent distributions over the next bisimilar states for all possible actions ~\citep{larsen1989bisim,Givan2003EquivalenceNA}. Formally:
\vspace{-1mm}
\begin{definition}[Bisimulation Relations~\citep{Givan2003EquivalenceNA}]
Given an MDP $\mathcal{M}$, an equivalence relation $B$ between states is a bisimulation relation if, for all states $s_i,s_j\in S$ that are equivalent under $B$ (denoted $s_i\equiv_Bs_j$) the following conditions hold:
\begin{alignat}{2}
     r(s_i,a )&\;=\; r(s_j,a ) 
    &&\quad \forall a \in U, \label{eq:bisim-discrete-r} \\
     P(G|s_i,a )&\;=\; P(G|s_j,a ) 
    &&\quad \forall a \in U, \quad \forall G\in S_B, \label{eq:bisim-discrete-p}
\end{alignat}
where $ S_B$ is the partition of $ S$ under the relation $B$ (the set of all groups $G$ of states equivalent under $B$), and $ P(G|s,a )=\sum_{s'\in G} P(s'|s,a ).$ (See \cref{app:erc} for a primer on concepts related to equivalence relations)
\end{definition}

Finding the coarsest bisimulation relation is known to be an NP-hard problem \citep{Givan2003EquivalenceNA}. Further, the exact partitioning induced from a bisimulation relation is generally impractical as it a very strict notion of equivalence and seldom leads to meaningful compression of the original MDP, this is especially true in continuous domains, where infinitesimal changes in the reward function or dynamics can break the bisimulation relation but still imply exploitable aggregation. Thus towards addressing this, Bisimulation Metrics~\citep{ferns2011contbisim,ferns2014bisim_metrics,castro20bisimulation,vanbreugel2001QuantitativeVerificationProbabilistic} relaxes the concept of exact bisimulation, and instead define a pseudometric space $( S, d)$, where a distance function $d: S\times S\mapsto\mathbb{R}_{\geq 0}$ measures the behavioral similarity between two states.
The bisimulation metric is formally defined as a convex combination of the reward difference added to the Wasserstein distance between transition distributions:
\begin{definition}[Bisimulation Metric]
\label{def:bisim_metric}
From Theorem 2.6 in \citep{ferns2011contbisim} with $c\in[0,1)$:
\begin{align}
d(s_i,s_j) 
\;&=\; \max_{a \in  U}\;
(1-c)\cdot | r_{s_i}^a  -  r_{s_j}^a | + c\cdot W_1( P_{s_i}^a , P_{s_j}^a ;d).
\label{eq:bisim_metric}
\end{align}
\end{definition}
\vspace{-2mm}$W$ refers to the Wasserstein-$p$ metric between two probability distributions $ P_i$ and $ P_j$, defined as
$W_p( P_i, P_j;d)=\inf_{\gamma'\in\Gamma( P_i, P_j)}[\int_{ S\times  S} d(s_i,s_j)^p\, \textnormal{d}\gamma'(s_i,s_j)]^{1/p}$,
where $\Gamma( P_i, P_j)$ is the set of all couplings of $ P_i$ and $ P_j$. The metric has intuitive interpretations depending on the exact value of $p$ when viewed from the dual perspective, for example $W_1( P_i, P_j;d)$ denotes the cost of transporting mass from distribution $P_i$ to another distribution $P_j$ where the cost is given by the distance metric $d$~\citep{optimaltransport}.
This is known as the earth-mover distance. 
The above definition can also be modified to include scenarios involving stochastic rewards, where a similar metric is chosen between reward distributions.
To account for state similarities arising from following a particular policy, the $\pi$-bisimulation metric \citep{castro20bisimulation} is similarly defined by fixing a policy $\pi$ and replacing the rewards and transitions used by their policy based expectations:
\begin{align}
d^\pi(s_i,s_j) 
\;&=\;
(1-c)\cdot | r_{s_i}^\pi -  r_{s_j}^\pi| + c\cdot W_1( P_{s_i}^\pi, P_{s_j}^\pi;d^\pi).
\label{eq:pi_bisim_metric}
\end{align}
In this work we will consider the max entropy RL framework as it ensures a unique optimal policy $\pi^*_{merl}$\footnote[3]{We will refer to it as $\pi^*$ in this work for brevity.}. Our goal is to leverage generalization and transfer obtained from informing the agent representation with the bisimulation similarity metric (\cref{eq:pi_bisim_metric}) under $\pi^*$.

\section{Methodology}
\label{method}
\begin{wrapfigure}{r}{0.6\linewidth}
    \centering
    \vspace{-35pt}
    \includegraphics[width=\linewidth]{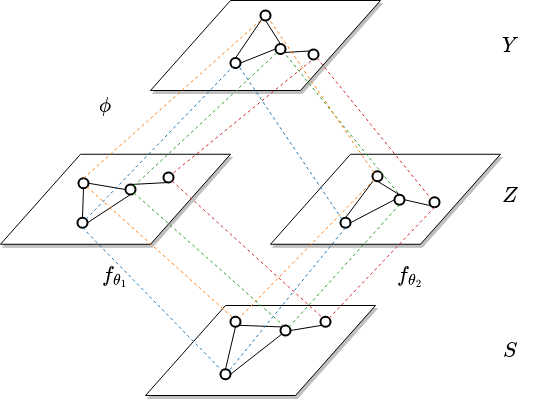}
    \caption{Learning representation invariant to observation shifts. Hollow circles represent states in the space, solid lines depict distances in the corresponding space, dashed lines depict equivalence across spaces tied by the colour.}
    \label{fig:representation}
    \vspace{-20pt}
\end{wrapfigure}
As previously discussed, it is important that agent policies in RL are robust to observation shifts for deployment in real world scenarios. In this work we wish to learn policies which can generalize well across the support set of the context distribution $P_\Theta$.
Our goal specifically would be to learn an effective representation function for the RL task set, $\phi: Z\times\Theta\mapsto{Y}$ which enables robust learning and deployment of autonomous agents to potentially unseen observation shifts (governed by a change in $\theta$), see ~\cref{fig:representation}. Under suitable notions of invertibility (\cref{as:block}),
the problem of generalizing across parameterized observation shifts (\cref{rocmdp}) lends itself naturally to the notion of MDP isomorphism (\cite{ravindran2004algebraic}, \smash{\cite{mahajan2017symmetry}}, see \cref{mdphomo} in \cref{app:proofs}). This is because given two contexts $\theta_i, \theta_j$, there is always a one to one mapping between the observations in $\mathcal{M}_{\theta_i} \iff \mathcal{M}_{\theta_j}$ as directed by the underlying state, this is illustrated in \cref{fig:representation}. This inter-context correspondence helps us inform the representation more efficiently.
Concretely, we specify the desiderata which the representation function $\phi$ must follow, as shown in \cref{fig:bisim_losses}:
\begin{itemize}
    \item \textbf{Base Bisimulation (BB)}: Given a $\theta \in \Theta$, the representation should accurately preserve bisimulation distances between states, thus providing robustness to unimportant noise in observations. Concretely $\forall s_i, s_j \in \mathcal{S}$:
    \begin{align}
        d(s_i, s_j) &= d_{Y}(\phi(f_\theta(s_i), \theta), \phi(f_\theta(s_j), \theta)),
    \end{align}
    where $d_{Y}$ is a metric on $Y$ (we use $Y=\mathbb{R}^m$ and L1 distance for our experiments).
    \item \textbf{Inter-context consistency (ICC)}: The representation should remain invariant under a fixed state as the context changes. Concretely: $\forall s \in \mathcal{S}$ and $\theta_1, \theta_2 \in \Theta$,
    \begin{align}
        d_{Y}(\phi(f_{\theta_1}(s), \theta_1), \phi(f_{\theta_2}(s), \theta_2))&=0.
    \end{align}
    \item \textbf{Cross consistency (CC)}: This requires that the representation distance between two states are consistent across observation shifts: 
    \begin{align}
        d(s_i, s_j)&=d_{Y}(\phi(f_{\theta_1}(s_i), \theta_1), \phi(f_{\theta_2}(s_j), \theta_2)),\\
        d(s_i, s_j)&=d_{Y}(\phi(f_{\theta_2}(s_i), \theta_2), \phi(f_{\theta_1}(s_j), \theta_1)).
    \end{align}
\end{itemize}
\begin{figure}[H]
    \centering
    \includegraphics[width=0.75\linewidth]{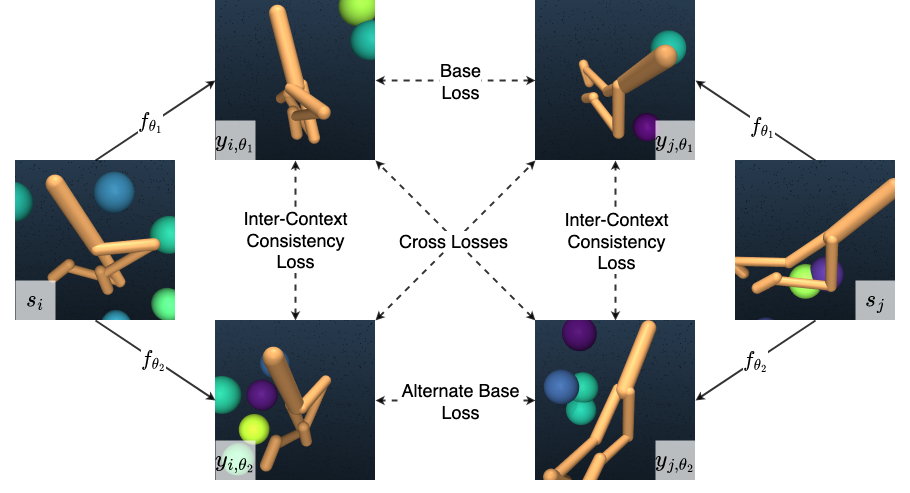}
    \caption{\small Various bisimulation losses corresponding to above desiderata. $s$ represents underlying state, $f_\theta$ the observation function and $y$ the corresponding observation, dashed lines represent bisimulation terms.}
    \label{fig:bisim_losses}
    \vspace{-10pt}
\end{figure}

\begin{wrapfigure}{R}{0.55\textwidth}
\vspace{-9mm}
\begin{minipage}{0.55\textwidth}
\begin{algorithm}[H]
    \begin{algorithmic}[1]
    \floatname{algorithm}{Procedure}
    \FOR {Time $t = 0$ to $\infty$}
    \STATE{Observe $z_{t}, \theta$}
    \STATE{Encode observation $y_t = \phi(z_{t}, \theta)$}
    \STATE{Execute action $a _t \sim \pi(y_t)$}
    \STATE{Record data: $\mathcal{D} \leftarrow \mathcal{D} \cup \{z_{t}, a _t, z_{t+1}, r_{t+1}\}$}
    \STATE {Sample batch $\mathcal{B}\sim\mathcal{D}$}
    \STATE{Train policy: $\mathbb{E}_{\mathcal{B}}[J^\pi]$}
    \STATE{Train encoder using pairwise loss: $\mathcal{L}_{rep}(\phi)$
    \algorithmiccomment{\cref{eq:bisim_loss}}}
    \STATE{\smash{Train dynamics: $J(\hat{P},\!\phi) \!=\! (\hat{P}(\phi(z_{t}, \theta),a_t) \!-\! {y}_{t+1})^2$\!\!\!\!}}
    \ENDFOR
    \end{algorithmic}
    \caption{Robust Conditional Bisimulation (RCB)}
    \label{alg:cbisim}
\end{algorithm}
\end{minipage}
\vspace{-3mm}
\end{wrapfigure}
\cref{fig:bisim_losses} depicts the above representation criteria for $2$ different contexts($\theta_1, \theta_2$) on the Mujoco control domain with 3D background objects acting as noise. Towards ensuring the above desiderata, we propose Robust Conditional Bisimulation (RCB) \cref{alg:cbisim}, a data-efficient approach to learn control policies from unstructured, high-dimensional observations. As evident from \cref{fig:bisim_losses}, for $n$ parallel simulation calls, our method captures \smash{${2n \choose 2}\sim O(n^2)$} interactions for representation learning using the above conditional bisimulation terms as opposed to $O(n)$ interactions in existing representation learning methods. For instance data augmentation methods based on contrastive learning (like \cite{oord2018representation, laskin_srinivas2020curl}) focus only on $(f_{\theta_1}(s), f_{\theta_2}(s))$ pairs whereas as plain bisimulation methods (like \cite{zhangLearningInvariantRepresentations2021}) focus only on $(f_{\theta}(s_1), f_{\theta}(s_2))$ pairs. This order of magnitude increase in utilization of metric information in RCB allows for fast and efficient convergence to an observation invariant representation space.

We combine the above three representation conditions into a sum of squared loss components. For this we sample pairs of experiences $i, j$ from the buffer along with base context $\theta_1$ and an alternate context $\theta_2$ both sampled from $P_\Theta$ at episode start. We next compute the embedding of the underlying states under the contexts and finally compute the representation loss term as follows:
\begin{align}
    \mathcal{L}_{rep}(\phi) =& \lambda_{base}\big[\big(|\overline{y}_{i,\theta_1} - \overline{y}_{j,\theta_1}|_1 - T_{i,j}\big)^2 + \big(|\overline{y}_{i,\theta_2} - \overline{y}_{j,\theta_2}|_1 - T_{i,j}\big)^2\big] + \\
    & \lambda_{icc}\big[ |\overline{y}_{i,\theta_1} - \overline{y}_{i,\theta_2}|_1^2 +
    |\overline{y}_{j,\theta_1} - \overline{y}_{j,\theta_2}|_1^2 \big]+\\
    & \lambda_{cc}\big[\big(|\overline{y}_{i,\theta_1} - \overline{y}_{j,\theta_2}|_1 - T_{i,j}\big)^2+
    \big(|\overline{y}_{i,\theta_2} - \overline{y}_{j,\theta_1}|_1 - T_{i,j}\big)^2\big],
    \label{eq:bisim_loss}
\end{align}
where we use the following notation: $y_{i,\theta} = \phi(f(s_i, \theta), \theta)$ with $\overline{y}_{i,\theta}$ representing embeddings with stopped gradient and the target bisimulation distance $T_{i,j} = |r_i-r_j|+\gamma W_2(\hat P(\cdot|y_{i,\theta_1}, a_i),\hat P(\cdot|y_{j,\theta_1}, a_j))$. The relative weights for the three loss terms are given by hyper-parameters $\lambda_{base}, \lambda_{icc}, \lambda_{cc}$ respectively. We use a setup similar to \cite{zhangLearningInvariantRepresentations2021} where we use a permuted batch of $\mathcal{B}$ for pairwise representation loss computation in step-$8$ of \cref{alg:cbisim}. Similarly we a probabilistic dynamics model $\hat{P}$ which outputs a Gaussian distribution. This allows for a simple to compute closed form $W_2$ metric which is used to replace the $W_1$ metric in the original formulation:
$W_2(\mathcal{N}(\mu_i,\Sigma_i),\,\mathcal{N}(\mu_j,\Sigma_j))^2 =||\mu_i-\mu_j||_2^2 + ||\Sigma_i^{1/2} \!-\! \Sigma_j^{1/2}||^2_\mathcal{F}$,
where $||\cdot||_\mathcal{F}$ is the Frobenius norm. \cref{fig:network_bisim} depicts the overall representation learning process. 
\begin{wrapfigure}{r}{0.6\linewidth}
    \centering
     \includegraphics[width=\linewidth]{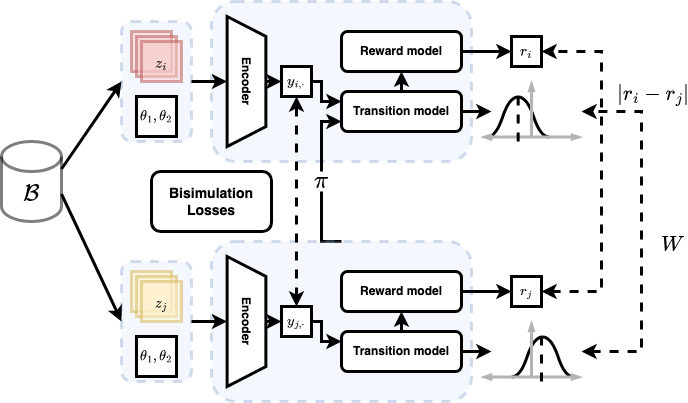}
    \caption{RCB Network architecture. }
    \label{fig:network_bisim}
    \vspace{-4mm}
\end{wrapfigure}
Finally, for the policy optimization part in step-$7$, we can use any max entropy policy gradient method. Access to simulator helps us translate a sampled batch from buffer into any randomly sampled contexts from which we can compute the various losses.
However, in general this technique can also be extended to non-simulator settings like data augmentation~\citep{laskin2020reinforcement}, this could be specially promising as the latter approaches currently only minimize representation distance between two views of same input and not the bisimulation distance which is more aligned with solving the RL task.

\section{Analysis}
We now discuss the important theoretical properties of our approach and study the generalization we can expect from learning representations under the conditional bisimuation framework. Proofs for the results can be found in \cref{app:condbisim}. The first result demonstrates the convergence of the $\pi^*$-bisimulation metric \cref{eq:pi_bisim_metric} on the joint input space $H\triangleq Z\times\Theta$ (We use the notation $h\triangleq(z, \theta)$ for a tuple in this space). We also overload the notion of policy($\pi$) to implicitly contain $\phi$ so that it can be viewed as operating on the joint space.
\begin{theorem}
\label{thm:bisim_fixed_point}
Let $\mathfrak{met}$ be the space of bounded pseudometrics on $Z\times\Theta$ and $\pi$ a policy that is continuously improving. Define $\mathcal{F}:\mathfrak{met} \mapsto \mathfrak{met}$ by
\begin{align}
\mathcal{F}(d,\pi)(h_i,h_j)=(1-c)|r^{\pi}_{h_i} - r^{\pi}_{h_j}| + c W(d)(P^{\pi}_{h_i},P^{\pi}_{h_j})
\end{align}
Then, $\forall c \in (0,1)$, $\mathcal{F}$ has a least fixed point $\tilde{d}$ which is a  $\pi^*$-bisimulation metric.
\end{theorem}

We next discuss an important assumption we need to make towards obtaining generalization results for the observation shifts. 
\begin{assumption}[Block structure]
We assume that $f_{\theta_1}(s_1) \cap f_{\theta_2}(s_2) \neq \emptyset \implies s_i=s_j, \forall \theta_1,\theta_2$ so that the observation map is invertible.
\label{as:block}
\end{assumption}
This means that the observation space $Z$ can be partitioned into disjoint blocks, each containing the support for a particular value of $s \in S$ \citep{du2019pcid}. This also ensures that inverse observation map $f_{\theta}^{-1}: Z\to S$ exists. Relaxing \cref{as:block} can break any guarantees obtainable on value function similarities arising from state similarity. This is because the same observation can get mapped to entirely different states in the latent MDP each with very different values, making the environment only partially observable. Note however that this requirement is not too restrictive, it is possible to consider added noise scenarios (both independent and correlated e.g. see ~\citep{zhangInvariantCausalPrediction2020}) which maintain identifiability of the state. Finally, many real-world task observations tend to satisfy this assumption for high dimensional scenarios: e.g. visual projection of non-degenerate objects under different viewing angles.

We next discuss the implications of having learnt an representation $\phi$ which approximately preserves the  $\pi^*$-bisimulation metric distances. 
\begin{theorem}[Aggregation value bound]
\label{thm:value_bound}
Given an MDP $\hat{\mathcal{M}}$ constructed by aggregating tuples $h$ of observation, context in an $\epsilon$-neighborhood of the representation space such that $\delta \triangleq \max_{s,s', \theta_i, \theta_j}||\phi(f_{\theta_i}(s), \theta_i)-\phi(f_{\theta_j}(s'), \theta_j)|-d_{S }(s,s')|$, where $d_{S }$ is a $\pi^*$-bisimulation metric on $S$. Further let $\hat \phi$ denote the map from any $h$ to these clusters, the optimal value functions for the two MDPs follow:
\begin{align}
|V^*(h) - \hat V^*(\hat\phi(h))| \leq \frac{2(\epsilon+\delta)}{(1-\gamma)(1-c)} \forall h\in Z\times\Theta
\end{align}
\end{theorem}

Note how the value estimate accuracy from aggregation is fundamentally bottle-necked by the representation learning error $\delta$, this means that even the finest partitions (which use small $\epsilon$) using $\phi$ will give value approximation only as good as the underlying representation.

We now state the Lipschitz continuity assumptions we use for further analysis. The first \cref{as:lipobs} concerns the change in observations $z$ as the context $\theta$ changes. Several natural domains like visual projections satisfy this.  
\begin{assumption}
$f$ is Lipschitz with coefficient $L_{\theta}^f$ with respect to (w.r.t.) $\theta$.
\label{as:lipobs}
\end{assumption}

Next, we assume that the representation map $\phi$ and the policy $\pi$ which conditions on the representations $y$ are also Lipschitz w.r.t. the inputs. This can be enforced in practice for example for deep neural networks approximators \citep{virmaux2018Lipschitz, gouk2021regularisation}. 
\begin{assumption}
$\phi$ is Lipschitz w.r.t. $z$ and $\theta$ with coefficients $L_{z}^{\phi}, L_{\theta}^{\phi}$ respectively. Similarly, $\pi$ is Lipschitz with coefficient $L_{y}^{\pi}$ where the distance metric on the policy space is $d_{TV}$ \footnote{note that $L_{y}^{\pi}$ has the effect of squeezing inflations caused by $L_{\theta}^f$ and $L_{z}^{\phi}$ as $d_{TV}$ is a bounded metric}, the total variation metric on space of action distributions $\mathcal{P}(U)$.
\label{as:lipp}
\end{assumption}

We now discuss the amount of generalization which we can expect when a policy assuming context $\theta_i$ is run on observation coming from the context $\theta_j$. This can happen for example in scenarios when a shift in observations happens like change in the calibration settings of an autonomous vehicle's sensors. We introduce the notation $\pi_{\theta_i\leftarrow \theta_j}$ to represent the policy obtained from sampling action w.r.t. the restriction $\pi_{\theta_i}$ but using observation inputs from the context $\theta_j$ (ie. $\pi(a|\phi(f_{\theta_j}(s),\theta_i))$). 
\begin{theorem}[Generalization to unseen context]
\label{th:gen}
Under \cref{as:lipobs}, \cref{as:lipp} we have that for any two contexts $\theta_i, \theta_j$:
\begin{align}
    |J^{\pi_{\theta_i}}-J^{\pi_{\theta_i\leftarrow \theta_j}}|\leq& \frac{1}{1-\gamma} E_{\substack{s\sim f_{\theta}^{-1}\rho^{\pi_{\theta_i}},\\ a\sim\pi_{\theta_i\leftarrow \theta_j}}}\Big[ A^{\pi_{\theta_i}}(s, a) + \frac{2\gamma A_{max}}{1-\gamma}L_{\theta}^{f}L_{z}^{\phi}L_{y}^{\pi}d_\Theta(\theta_i, \theta_j) \Big]
\end{align}
where $A_{max}\triangleq \max_s |E_{a\sim\pi_{\theta_i\leftarrow \theta_j}}[A^{\pi_{\theta_i}}(s, a)]|$ and $d_\Theta$ is a metric on the context space.
\end{theorem}
\cref{th:gen} gives us the upper bound on the deviation of the expected returns when the agent expects an environment with context $\theta_i$ but is actually deployed in an environment with context $\theta_j$. 

We next discuss the important performance transfer scenarios when the simulator used for training a policy is not exact. These bounds are useful for situations where it is required to access tolerance of agent performance w.r.t. situations like sim to real deployment. Our first result addresses the setting where the simulator dynamics is not exact w.r.t. the real world, introducing errors $\epsilon_{R}, \epsilon_{P}$. 
\begin{theorem}[Simulator fidelity bound]
\label{th:simfid}
For an approximately correct simulator ($\hat r, \hat P$) such that $\max_{s,a}|\hat r(s,a)-r(s,a)|\leq \epsilon_R$ and $\max_{s,a} d_{TV}(\hat P(s, a),P(s,a))\leq \epsilon_P$ we have for any policy $\pi$:
\begin{align}
|J^\pi - \hat J^\pi| \leq \frac{\epsilon_R}{(1-\gamma)} + \frac{\gamma \epsilon_P R_{max}}{(1-\gamma)^2}
\end{align}
\end{theorem}
Next, we consider the case where in addition to the latent transition and reward dynamics, the simulator emission function $\hat f$ is also approximate. 
Let $\epsilon_f  \triangleq \max_{s, \theta} d_{Y}(\phi(\hat f_\theta(s)), \phi(f_\theta(s)))$.
We are interested in the setting where the policy learns from an approximate simulator ($\hat r, \hat P, \hat f$) but the resultant learnt policy is deployed in the actual world $( R,  P, f)$. Note that this is a common practical setting as most simulators, even after knowing the actual underlying state, cannot completely capture the richness in the observations found in the real world. The below result relates the simulator policy performance ($\hat f$) to the one obtained by running the simulator policy on real observations ($f$). 
\begin{theorem}[Complete simulator fidelity bound]
\label{th:simfid_all}
For an approximately correct simulator ($\hat r, \hat P, \hat f$) such that $\max_{s,a}|\hat r(s,a)-r(s,a)|\leq \epsilon_R$, $\max_{s,a} d_{TV}(\hat P(s, a),P(s,a))\leq \epsilon_P$ and $\epsilon_f  \triangleq \max_{s, \theta} d_{Y}(\phi(\hat f_\theta(s)), \phi(f_\theta(s)))$, we have for any policy $\pi$:
\begin{align}
|J^{\pi_{\hat{f}\leftarrow f}} - \hat J^{\pi_{\hat{f}}}| &\leq \frac{\epsilon_R}{(1-\gamma)} + \frac{\gamma \epsilon_P R_{max}}{(1-\gamma)^2}+ \frac{1}{1-\gamma} E_{\substack{s\sim \hat f^{-1}\rho^{\pi_{\hat{f}}},\\ a\sim\pi_{\hat{f}\leftarrow f}}}\Big[ A^{\pi_{\hat{f}}}(s, a)+ \frac{2\gamma A_{max}}{1-\gamma}L_{y}^{\pi}\epsilon_f \Big].
\end{align}
\end{theorem}
$\pi_{\hat{f}\leftarrow f}$ represents the sampling of actions from $\pi_{\hat{f}}$ but using the observations obtained under the (real world) observation function $f$.  Thus, the above two results (Theorems \ref{th:simfid} and \ref{th:simfid_all}) are particularly useful for the realistic scenario where we have imprecise simulation dynamics.

\section{Experiments}
\label{sec:exps}
We perform experiments towards understanding whether our method Robust Conditional Bisimulation (RCB) helps learn representations which generalize better to observation shifts. Towards this, we use the DeepMind control suite (DMC, \cite{tassaDeepMindControlSuite2018}) which uses Mujoco \citep{todorov2012mujoco} as the base simulator. We create new tasks for various agent morphologies where we learn to control the agent using image based input. Further, we also modify the simulator to have 3D spheres randomly bouncing in the environment, which contribute towards noise (we call this Modified-DMC). Note that this noise setting is harder than the simple-distractor setting in \cite{zhangLearningInvariantRepresentations2021} as the agent has to learn to model 3D noise across different visual perspectives (see \cref{fig:bisim_losses}). We use two baselines for comparison: 
\begin{enumerate}
    \item DeepMDP~\citep{gelada2019deepmdp} which uses reward and forward dynamics predictability for learning a latent representation space.
    \item A reconstruction based agent which uses a reward model and an image reconstruction based emission model to inform the representation.
\end{enumerate} We use SAC \citep{haarnoja2018soft} as the base algorithm for optimizing the MERL objective in \cref{alg:cbisim}. The architecture for common modules is kept similar across the methods. For fair comparison, we ensure that all the methods get equal access to the simulator experience and augment the representation learning objective for baselines with any extra simulator calls. Additional experimental setup details can be found in \cref{app:aed}.

\textbf{Modified-DMC:} For testing the ability to generalize across observation shifts, we use a uniform distribution over the range $P_\Theta = \mathcal{U}(-\pi/4, \pi/4)$ for the camera angle. At the beginning of each episode, we sample a camera angle context from $P_\Theta$, the agents must adapt to changing image perspectives across training and evaluation. For evaluation, we use a fixed set of camera angles: $\{-\pi/4,-\pi/8, 0, \pi/8, \pi/4\}$ over which we compute the agent performance during the evaluation phase and report the average across the angles as the performance metric. \cref{fig:bisim_exps} gives the evaluation performance plots for the agents on five different scenarios averaged over five seeds with one standard error shaded (our method \textcolor{blue}{RCB} in blue, \textcolor[rgb]{0,0.7,0}{DeepMDP} in green, \textcolor{red}{Reconstruction} in red). We see that RCB performs significantly better than the baseline agents on all the scenarios. RCB consistently achieves higher performance across the walker tasks (\Crefrange{fig:wsb}{fig:wrb}). We also note that the performance for Reconstruction worsens as the task becomes more dynamic, we hypothesize this is due to the lack of focus on the core features of the observation which influence the reward and dynamics. We observe a similar trend on the cheetah domain (\cref{fig:crb}) which is slightly easier than walker run. DeepMDP is often unable to perform satisfactorily in the training budget, we posit that this happens as it does not use any inter-context information to inform its representation. Thus, while it may learn close distance embeddings for a fixed context, the embeddings fare poorly across the contexts. RCB alleviates this problem by leveraging both the ICC and cross-consistency objectives in its formulation. We also note that generalization for the hopper domain (\cref{fig:hhb}) while doing pixel based control is especially hard given the environment stochasticity and the added background noise. 
\begin{figure}[H]
    \centering
    \vspace{-10pt}
    \subfigure[Walker Stand]{
        \includegraphics[width=0.31\textwidth]{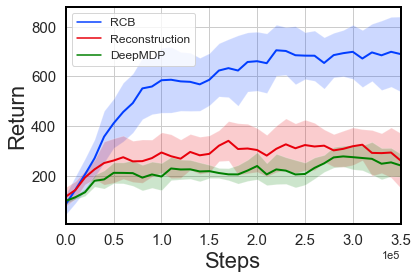}
        \label{fig:wsb}
    }
    \vspace{-5pt}
    \subfigure[Walker Walk]{
        \includegraphics[width=0.31\textwidth]{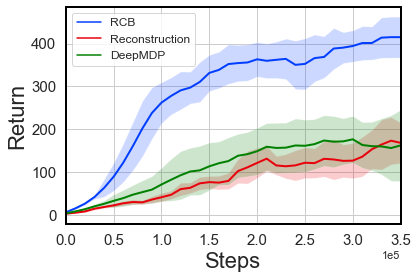}
        \label{fig:wwb}
    }
    \subfigure[Walker Run]{
        \includegraphics[width=0.31\textwidth]{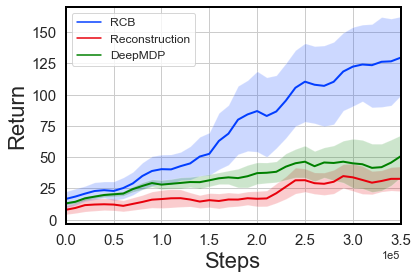}
        \label{fig:wrb}
    }
    \subfigure[Cheetah Run]{
        \includegraphics[width=0.31\textwidth]{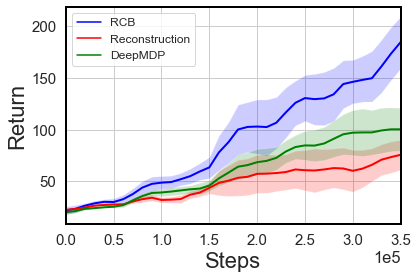}
        \label{fig:crb}
    }
        \vspace{-5pt}
    \subfigure[Hopper Hop]{
        \includegraphics[width=0.31\textwidth]{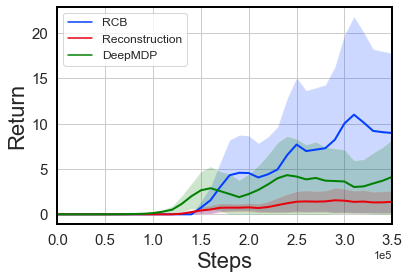}
        \label{fig:hhb}
    }
    \caption{Empirical results on modified-DMC observation generalization tasks for different methods: \textcolor{blue}{RCB} (our method), \textcolor[rgb]{0,0.7,0}{DeepMDP}, and \textcolor{red}{Reconstruction}.\label{fig:bisim_exps}}
\vspace{-15pt}
\end{figure}
\textbf{Out-of-distribution Generalization: }To test the ability of the algorithms in dealing with unseen observation contexts during test time, we train on Modified-DMC where we use a uniform distribution over the range $P_\Theta = \mathcal{U}(-3\pi/16, 3\pi/16)$ for the camera angle and test on the unseen $\{-\pi/4, \pi/4\}$ angles. \cref{fig:ood} gives the performance on the unseen angles for walker walk domain across the algorithms. Once again we note that RCB is able to better generalize to the unseen context due to its learning of a more accurate representation space using the inter-context objectives (\textbf{ICC} and \textbf{CC} terms in \cref{method}).
\begin{figure}[htb!]
\vspace{-10pt}
    \centering
    \subfigure[Out-of-distribution ($P_\Theta$)]{
        \includegraphics[width=0.31\textwidth]{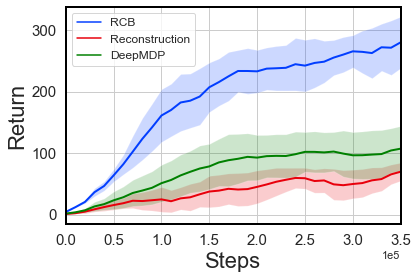}
        \label{fig:ood}
    }
    \subfigure[Ablation 1]{
        \includegraphics[width=0.31\textwidth]{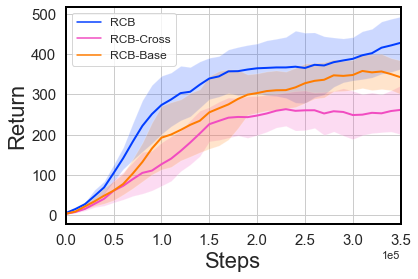}
        \label{fig:ab1}
    }
    \subfigure[Ablation 2]{
        \includegraphics[width=0.31\textwidth]{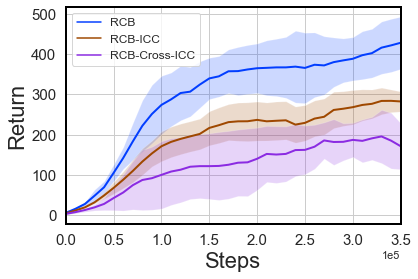}
        \label{fig:ab2}
    }
    \vspace{-5pt}
    \caption{Out-of-distribution generalization and ablations} %
\vspace{-10pt}
    \label{fig:oodab}
\end{figure}
\newline \textbf{Ablations: } To understand the effects of the different bisimulation loss components, we perform ablations removing each component. In \cref{fig:ab1} we remove the base (\textcolor{orange}{RCB-Base}) and cross consistency (\textcolor[rgb]{0.9450980392156862, 0.2980392156862745, 0.7568627450980392}{RCB-Cross}) bisimulation terms. We see that removing the cross term has a bigger effect on performance. We believe this is because the cross-bisimulation term has a stronger anchoring effect as it also implicitly accounts for both the base and inter-context terms (see \cref{fig:bisim_losses}). Next, in \cref{fig:ab2} we remove the inter-context consistency term (\textcolor{brown}{RCB-ICC}) and both the inter-context consistency and cross consistency terms (\textcolor{purple}{RCB-Cross-ICC}). We notice a slight decrease in performance arising from dropping the inter context consistency loss. The \textcolor{purple}{RCB-Cross-ICC} ablation is similar to DBC \citep{zhangLearningInvariantRepresentations2021} as it only contains base bisimulation losses (regular and alternate). We observe a significant decrease in performance in this latter ablation as we drop all the inter-context bisimulation terms helpful in generalization across the contexts. Thus it is important to ensure explicit alignment across the representations for the richly observed MDPs defined by different $\theta$ context when desiring good generalization across observation shifts.

\section{Related Work}
\vspace{-2mm}
State abstraction in MDPs has been researched from various perspectives including notions which aggregate state based on policy, values, action-values and dynamics ~\citep{li2006stateabs}. Bisimulation is the strictest form of abstraction based on MDP dynamics~\citep{larsen1989bisim, Givan2003EquivalenceNA}. Bisimulation metrics were introduced to relax the notion of exact bisimulation for practical applicability \citep{ferns2011contbisim}. \cite{castro20bisimulation} propose method for efficient computation of on policy variant of bisimulation metrics.
DBC \citep{zhangLearningInvariantRepresentations2021} use bisimulation metrics to learn task relevant features which are robust to noise in the environment. 
They learn to tie together states distinguishable only by task irrelevant noise using bisimulation for learning a representation.
We use the bisimulation framework to learn a representation which can \textit{invert} the change in observation space caused by the varying context, and can be seen as abstracting across the group of isomorphic MDPs indexed by the context. 
We also provide the first generalization bounds for this setting with important practical applications like sim to real transfer.
MDP homomorphism \citep{ravindran2004algebraic} is the principled framework of studying structural similarities across MDPs, this naturally extends the idea of state abstraction and opens the the way to leverage abstract similarities on a much broader scope. \cite{taylor2009bounding} propose lax probabilistic bisimulation using MDP homomorphisms. \cite{mahajan2017symmetry} use MDP isomorphism for learning symmetries for sample efficient reinforcement learning. \citep{mahajan2022generalization} provide combinatorial generalization bounds for the contextual multi agent setting.

\vspace{-2mm}
Representation learning for RL on high dimensional inputs has been studied using other methods in addition to bisimulation. \cite{Lange2012Autonomous} use reconstruction of image inputs using auto-encoders for learning a latent control state. This was later extended to include modelling of the MDP dynamics \citep{watter2015embed, hafner2018planet}. \cite{gelada2019deepmdp} use a latent dynamics model approach for control and show its relation to bisimulation.
Data Augmentation methods like \cite{laskin2020reinforcement} use various image transformations on agent observations for data efficient learning of policies for pixel based control. \cite{laskin_srinivas2020curl} use random crops on image data to be used under a contrastive based framework for representation learning. \cite{kostrikov2020image} use random image translations for regularising reinforcement learning from images by using multiple shifts to robustly estimate value function loss and targets. \cite{oord2018representation,chen2020simclr} use self-supervised contrastive approach to learn representations by enforcing similarity constraints between data points.

\vspace{-2mm}
Robust RL considers rewards maximization under adversarially varying dynamics for the environment. \cite{pinto2017robust} use a two agent zero-sum game to model adversarial noise towards learning robust policies. Similarly, \cite{stantonrobust} inject noise in the state space and optimise for a minimax problem for robustness, and \cite{tessler2019action} study the robustness problem under action perturbations. Rather, we discuss the setting of adapting to potentially unseen deployment scenarios and provide theoretical guarantees for the policy transfer. \citep{openendedlearningteam2021openended} use dynamic curricula for learning robust agent policies.
\cite{zhao2020sim} compile the various methods used in sim-to-real settings.
Domain randomization, particularly used in robotic vision tasks including object localization~\citep{tobin2017domain}, object detection~\citep{tremblay2018training}, pose estimation~\citep{sundermeyer2018implicit}, and semantic segmentation~\citep{yue2019domain}, varies the training data from simulator across properties like textures, lighting, and camera positions. While domain randomization aims to provide enough simulated variability of the parameters at training time to ensure the model is able to generalize to potentially unseen settings during test, it is often insufficient for getting good results on control tasks due to unutilized information of task structure. 
\vspace{-2mm}
\section{Conclusions \& Future Work} 
\vspace{-2mm}
In this we work we explored how bisimulation can be used to learn representation for RL towards generalization in complex high dimensional environment like visual inputs. We specially focused on learning policies invariant to observation shifts a problem which has several applications in the real world. Further, we analysed the theory of learning under the framework of conditional bisimualtion and proposed novel bounds characterizing state abstraction and generalization in this setting. Of particular importance were the results relating to performance guarantees across observation shifts when learning on a simulator. Finally, we evaluated our method on the modified DM-contol domain and showed its efficacy in comparison to the baseline approach. A current limitation of our theoretical analysis is that it requires invertability (\cref{as:block}), we aim to relax this in future work using notions of approximate invertability. Another limitation is that the algorithms used in experiments presently take the context vector as input, we aim to replace this with automatic context detection using unsupervised methods in future.

\bibliography{main}
\bibliographystyle{collas2023_conference}

\newpage
\appendix
\section{Additional Definitions and Proofs}
\label{app:condbisim}
\addtocounter{theorem}{-5}
\subsection{Equivalence relations and classes}
\label{app:erc}
We first briefly mention some of the concepts from abstract algebra used in motivating state similarity in MDPs. 
\begin{definition}
	A binary relation $\mathcal{R}$ on a set $\mathcal{S}$ is given by $\mathcal{R} \subseteq \mathcal{S}\times\mathcal{S}$ 
\end{definition}

\begin{definition}
	$\mathcal{R}$ is symmetric if $\mathcal{R}(a,b)\implies \mathcal{R}(b,a)$
\end{definition}

\begin{definition}
	$\mathcal{R}$ is reflexive if $\mathcal{R}(a,a), \forall a\in \mathcal{S}$
\end{definition}

\begin{definition}
	$\mathcal{R}$ is transitive if $\mathcal{R}(a,b) \land \mathcal{R}(b,c) \implies \mathcal{R}(a,c)$
\end{definition}

\begin{definition}
	$\mathcal{R}$ is equivalence if its reflexive, symmetric and transitive.
\end{definition}

\begin{definition}
	$\mathcal{P}\triangleq \{\mathcal{C}_i\}$ is a partition of a set $\mathcal{S}$ if $ \mathcal{S} = \cup_i \mathcal{C}_i$ and $\mathcal{C}_i \cap  \mathcal{C}_j$ is empty if $i\neq j$.
\end{definition}

\begin{definition}
	If $\mathcal{R}$ is an equivalence relation on $\mathcal{S}$, then $\mathcal{S}$ can be partitioned into equivalence classes with $\mathcal{P}(\mathcal{R},\mathcal{S})\triangleq \{\mathcal{C}_i\}$, where $\mathcal{C}_i \subseteq \mathcal{S}, \forall a,b \in \mathcal{C}_i\implies\mathcal{R}(a,b)$ and $ \mathcal{C}_i \cap  \mathcal{C}_j$ is empty if $i\neq j$.
\end{definition}

\begin{definition}
For partitions $\mathcal{P}_1$ and $\mathcal{P}_2$, $\mathcal{P}_1$ is a filtrate of $\mathcal{P}_2$ if $\forall \mathcal{C}_i \in \mathcal{P}_2, \exists \mathcal{D}_j \in \mathcal{P}_1$ s.t. $\mathcal{C}_i = \cup_j \mathcal{D}_j$  
\end{definition}

\begin{definition}
	$\mathcal{P}_c$	is the coarsest partition induced by $\mathcal{R}$ if $\forall$ valid partitions $\mathcal{P}$ under $\mathcal{R}$, $\mathcal{P}$ is a filtrate of $\mathcal{P}_c$
\end{definition}

\subsection{Proofs}
\label{app:proofs}
We first revisit the concept of MDP homomorphisms \citep{ravindran2004algebraic} which we will use for establishing important results concerning the conditional bisimulation framework.
\begin{definition}[MDP homomorphism \citep{ravindran2004algebraic}]
\label{mdphomo}
Let $\Psi \subset S\times U$ is the set of admissible state-action pairs. MDP homomorphism $\mathcal{H}$ from $M=\langle S,U,\Psi,P,r,\gamma, \rho \rangle$ to $M'=\langle S',U',\Psi',P',r',\gamma, \rho' \rangle$ is defined as a surjection $\mathcal{H}:\Psi\rightarrow \Psi'$, which is itself defined by a tuple of surjections $\langle f,\{g_s,s\in S\}\rangle$. In particular, $\mathcal{H}((s,a)):=(f(s),g_s(a))$, with $f:S\rightarrow S'$ and $g_s:A_s\rightarrow A_{f(s)}'$, which satisfies two requirements: Firstly it preserves the reward function: 
\begin{align}
r'(f(s),g_s(a))=r(s,a)
\end{align}
and secondly it commutes with transition dynamics of $M$:  
\begin{align}
P'(f(s), g_s(a),f(s'))=P(s,a,[s']_{B_{\mathcal{H}|S}})     
\end{align}
\end{definition}

Here we use the notation $[\cdot ]_{B_{\mathcal{H}|S}}$ to denote the \textit{projection} of equivalence classes $B$ that partition $\Psi$ under the relation $\mathcal{H}((s,a))=(s',a')$ on to $S$. 
Isomorphisms $\chi:\Psi \rightarrow \Psi'$ can then be formally defined as homomorphisms between $M, M'$ that completely preserve the system dynamics with the underlying functions $f,g_s$ being bijective. 

\begin{theorem}
Let $\mathfrak{met}$ be the space of bounded pseudometrics on $Z\times\Theta$ and $\pi$ a policy that is continuously improving. Define $\mathcal{F}:\mathfrak{met} \mapsto \mathfrak{met}$ by
\begin{align}
\mathcal{F}(d,\pi)(h_i,h_j)=(1-c)|r^{\pi}_{h_i} - r^{\pi}_{h_j}| + c W(d)(P^{\pi}_{h_i},P^{\pi}_{h_j})
\end{align}
Then, $\forall c \in (0,1)$, $\mathcal{F}$ has a least fixed point $\tilde{d}$ which is a  $\pi^*$-bisimulation metric.
\end{theorem}
\begin{proof}
First, consider the super-MDP over the unified state space $H\triangleq Z\times\Theta$, $\mathcal{M}_{super}\triangleq \left\langle H,  U,  P_H,  r_H, \gamma, \rho_H \right\rangle$,  where the $H$ subscripted distributions implicitly account for $f, P_\Theta, \rho$. Similarly, let $\mathcal{M}_\theta$ be the MDP obtained by restricting the context to a particular value $\theta$ and $\mathcal{M}_{base}\triangleq \left\langle S,  U,  P,  r, \gamma, \rho\right\rangle$. We have that under \cref{as:block} $\mathcal{M}_{\theta}$ and $\mathcal{M}_{base}$ are isomorphic and all of $\mathcal{M}_{super}$, $\mathcal{M}_{\theta}$ and $\mathcal{M}_{base}$ are homomorphic \citep{ravindran2004algebraic, mahajan2017symmetry}.
Thus we can map the policy dynamics in the super-MDP exactly to the base MDP with states $S$. We now directly apply metric convergence result of Theorem 1 in \citep{zhangLearningInvariantRepresentations2021} on the representation space $Y$, 
thus showing that the $\pi$ bisimulation metric converges after repeated applications of the operator $\mathcal{F}$.
\end{proof}

\begin{theorem}[Aggregation value bound]
Given an MDP $\hat{\mathcal{M}}$ constructed by aggregating tuples $h$ of observation, context in an $\epsilon$-neighborhood of the representation space such that $\delta \triangleq \max_{s,s', \theta_i, \theta_j}||\phi(f_{\theta_i}(s), \theta_i)-\phi(f_{\theta_j}(s'), \theta_j)|-d_{S }(s,s')|$, where $d_{S }$ is a $\pi^*$-bisimulation metric on $S$. Further let $\hat \phi$ denote the map from any $h$ to these clusters, the optimal value functions for the two MDPs follow:
\begin{align}
|V^*(h) - \hat V^*(\hat\phi(h))| \leq \frac{2(\epsilon+\delta)}{(1-\gamma)(1-c)} \forall h\in Z\times\Theta
\end{align}
\end{theorem}
\begin{proof}
We use a proof strategy similar to \citep{zhangLearningInvariantRepresentations2021}.
We have that every $\theta$ restriction of $\mathcal{M}_{super}$ is isomorphic to $\mathcal{M}_{base}$ from the above proof. By direct application of Theorem 5.2 in \citep{ferns2004bisimulation} on the MDP $\mathcal{M}_{super}$ for any $h\in Z\times \Theta$:
\begin{align}
    (1-c)|V^*(h) - \hat V^*(\hat \phi(h))| \leq g(s,\tilde{d}) + \frac{\gamma}{1-\gamma}\max_{s'\in \mathcal{S}}g(s',\tilde{d})
\end{align}
where $g$ is the average distance between a state and all other states in its equivalence class under the bisimulation metric $\tilde{d}$. Substituting $g$ with the $\epsilon$-neighborhood ball, and accounting for $\delta$, the error of the representation w.r.t. the metric for each cluster gives us:
\begin{align}
    (1-c)|V^*(h) - \hat V^*(\hat \phi(h))| &\leq 2(\epsilon+\delta) + \frac{\gamma}{1-\gamma}2(\epsilon+\delta) \\
    |V^*(h) - \hat V^*(\hat \phi(h))| &\leq \frac{1}{1-c}\Big(2(\epsilon+\delta) + \frac{\gamma}{1-\gamma}2(\epsilon+\delta)\Big) \\
    &=\frac{2(\epsilon+\delta)}{(1-\gamma)(1-c)}.    
\end{align}
\end{proof}

\begin{lemma}
\label{le:lipcomp}
Let $f: X\to Y$, $g: Y\to Z$ be two functions with Lipschitz constants $L_1$ and $L_2$ respectively, then $g(f(\cdot))$ is Lipschitz with $L_1\cdot L_2$
\end{lemma}
\begin{proof}
Computing deviations for the various functions and using the definition of Lipschitzness, we have that:
\begin{align}
df &\leq L_1 dx\\
dg &\leq L_2 dy = L_2 df\\
\implies dg &\leq L_2 L_1 dx
\end{align}
Thus $g(f(\cdot))$ is Lipschitz with $L_1\cdot L_2$ w.r.t. $X$.
\end{proof}

\begin{theorem}[Generalization to unseen context]
Under \cref{as:lipobs}, \cref{as:lipp} we have that for any two contexts $\theta_i, \theta_j$:
\begin{align}
    |J^{\pi_{\theta_i}}-J^{\pi_{\theta_i\leftarrow \theta_j}}|\leq& \frac{1}{1-\gamma} E_{\substack{s\sim f_{\theta}^{-1}\rho^{\pi_{\theta_i}},\\ a\sim\pi_{\theta_i\leftarrow \theta_j}}}\Big[ A^{\pi_{\theta_i}}(s, a) + \frac{2\gamma A_{max}}{1-\gamma}L_{\theta}^{f}L_{z}^{\phi}L_{y}^{\pi}d_\Theta(\theta_i, \theta_j) \Big]
\end{align}
where $A_{max}\triangleq \max_s |E_{a\sim\pi_{\theta_i\leftarrow \theta_j}}[A^{\pi_{\theta_i}}(s, a)]|$ and $d_\Theta$ is a metric on the context space.
\end{theorem}
\begin{proof}
Under \cref{as:lipobs}, \cref{as:lipp}, we have that $d_{TV}(\pi_{\theta_i}(s), \pi_{\theta_j}(s))\leq L_{\theta}^{f}L_{o}^{\phi}L_{y}^{\pi}d_\Theta(\theta_i, \theta_j)$ for all underlying states $s\in S, \theta_i, \theta_j \in \Theta$, by repeated application of \cref{le:lipcomp}. We next apply Corollary 1 from \citep{achiam2017constrained} that uses the bound for performance difference as a function of policy TV distance giving us the result.
\end{proof}

\begin{theorem}[Simulator fidelity bound]
For an approximately correct simulator ($\hat r, \hat P$) such that $\max_{s,a}|\hat r(s,a)-r(s,a)|\leq \epsilon_R$ and $\max_{s,a} d_{TV}(\hat P(s, a),P(s,a))\leq \epsilon_P$ we have for any policy $\pi$:
\begin{align}
|J^\pi - \hat J^\pi| \leq \frac{\epsilon_R}{(1-\gamma)} + \frac{\gamma \epsilon_P R_{max}}{(1-\gamma)^2}
\end{align}
\end{theorem}
\begin{proof}
We proceed similar to \citep{jiang2018notes} for proving the policy value bound.
Let us consider the base MDP $\mathcal{M}_{base}$ as defined above. For any projected policy $\pi$ here, the value function satisfies $\forall s\in S$:
\begin{align}
&|\hat V^\pi(s)-V^\pi(s)|\\
&\leq |(\hat r(s, \pi) + \gamma \langle \hat P(s,\pi), \hat V^\pi \rangle) - ( r(s, \pi) + \gamma \langle P(s,\pi),  V^\pi \rangle) |\\
&\leq \epsilon_R+ \gamma|\langle \hat P(s,\pi), \hat V^\pi \rangle - \langle P(s,\pi),  V^\pi \rangle|\\
&\leq \epsilon_R+ \gamma|\langle \hat P(s,\pi), \hat V^\pi \rangle-\langle P(s,\pi),  \hat V^\pi \rangle+ \langle P(s,\pi),  \hat V^\pi \rangle- \langle P(s,\pi),  V^\pi \rangle|\\
&\leq \epsilon_R+ \gamma[|\langle \hat P(s,\pi)-P(s,\pi), \hat V^\pi \rangle|+|\hat V^\pi-V^\pi|_\infty ]\\
&\leq \epsilon_R+ \gamma[|\langle \hat P(s,\pi)-P(s,\pi), \hat V^\pi - \frac{R_{max}}{2(1-\gamma)}\mathbf{1} \rangle|+|\hat V^\pi-V^\pi|_\infty ]\\
&\leq \epsilon_R+ \gamma[| \hat P(s,\pi)-P(s,\pi)|_1 |\hat V^\pi - \frac{R_{max}}{2(1-\gamma)}\mathbf{1}|_\infty+|\hat V^\pi-V^\pi|_\infty ]\\
&\leq \epsilon_R+ \gamma[\frac{\epsilon_P R_{max}}{(1-\gamma)}+|\hat V^\pi-V^\pi|_\infty ]
\end{align}
Here we have viewed the probability transitions and values as vectors. The use of baseline $\frac{R_{max}}{2(1-\gamma)}$ helps tighten the bound by centering the values. We use the definition of TV in last step.
As the bound for all $s\in S$ we get after rearranging:
\begin{align}
|V^\pi - \hat V^\pi|_\infty\leq \frac{\epsilon_R}{1-\gamma}+ \frac{\gamma\epsilon_P R_{max}}{(1-\gamma)^2}
\end{align}
Finally, as $J^\pi$ is a convex combination of $V^\pi$ w.r.t. $\rho$, we can use the above bound to prove the result.
\end{proof}

\begin{theorem}[Complete simulator fidelity bound]
For an approximately correct simulator ($\hat r, \hat P, \hat f$) such that $\max_{s,a}|\hat r(s,a)-r(s,a)|\leq \epsilon_R$, $\max_{s,a} d_{TV}(\hat P(s, a),P(s,a))\leq \epsilon_P$ and $\epsilon_f  \triangleq \max_{s, \theta} d_{Y}(\phi(\hat f_\theta(s)), \phi(f_\theta(s)))$, we have for any policy $\pi$:
\begin{align}
|J^{\pi_{\hat{f}\leftarrow f}} - \hat J^{\pi_{\hat{f}}}| &\leq \frac{\epsilon_R}{(1-\gamma)} + \frac{\gamma \epsilon_P R_{max}}{(1-\gamma)^2}+ \frac{1}{1-\gamma} E_{\substack{s\sim \hat f^{-1}\rho^{\pi_{\hat{f}}},\\ a\sim\pi_{\hat{f}\leftarrow f}}}\Big[ A^{\pi_{\hat{f}}}(s, a)+ \frac{2\gamma A_{max}}{1-\gamma}L_{y}^{\pi}\epsilon_f \Big]
\end{align}
\end{theorem}
\begin{proof}
We consider an intermediate simulator $(\hat r, \hat P, f)$ which has the same reward and transition as the original simulator ($\hat R, \hat P$) but uses an exact observation function $f$ (we use $\tilde{}$ to represent quantities associated with this simulator). We can now decompose the difference bound as:
\begin{align}
|J^{\pi_{\hat{f}\leftarrow f}} - \hat J^{\pi_{\hat{f}}}| &\leq |J^{\pi_{\hat{f}\leftarrow f}} - \tilde{J}^{\pi_{\hat{f}\leftarrow f}}|+|\tilde{J}^{\pi_{\hat{f}\leftarrow f}} - \hat J^{\pi_{\hat{f}}}|
\end{align}
Next we have that the $d_{TV}(\pi_{\hat{f}\leftarrow f}, \pi_{\hat{f}})\leq L_{y}^{\pi}\epsilon_f$.
Reasoning similarly to \cref{th:gen} for the right term of RHS which gives an upper bound using the TV difference. Finally also applying \cref{th:simfid} on the left term of RHS we get the theorem's result.
\end{proof}

\section{Additional experimental details}
\label{app:aed}
\subsection{Architecture details}
We use separate deep networks for actor, critic, transition and reward models. The encoder network for each used $32$ filters and a $50$ feature dimensions. The actor and critic models each used an MLP trunk of $4$ layers and $1024$ hidden dimensions on top of the encoder. The reward model used MLP trunk of $2$ MLP layers and $512$ hidden dimensions on top of the encoder. The transition model type used was a mixture of Gaussians of ensemble size 5. Each component in the transition ensemble uses a $2$ MLP layers of $768$ hidden dimensions on top of the encoder with the final layer bifurcating for a value for mean and standard deviation per feature dimension. Layer normalization was used for the reward and transition models. Target networks were used for value estimates and were updated every $4$ epochs. Relu non-linearity was used for the networks. We exponentially anneal the representation loss with weight $(1.8 - 0.8*2^{\frac{\text{steps}}{\text{total steps}}})$.
We use identical architectures for the overlapping components of the baselines (Reconstruction and DeepMDP). The reconstruction agent uses an image decoder with an MLP followed by $2$ deconvolution layers with the intermediate layer using $32$ filters. Adam optimizer was used for training the parameters of the networks used. Grid search was used for tuning the hyperparameters.
Our code is based on implementation by \citep{zhangLearningInvariantRepresentations2021} for their work. Each seed takes around 4 days to run on an Nvidia V100 GPU.

\subsection{Hyper-parameters used: Conditional bisimulation }\am{done}
\begin{table}[h!]
\caption{Hyper-parameters used: Conditional bisimulation}\label{tab:hyper_bisim}
\vspace{0.3cm}
	\begin{center}
		\begin{tabular}{lc}
			\textbf{PARAMETER }&\textbf{VALUE} \\
			\hline \\
			$\lambda_{base}$ & 0.24 \\
			$\lambda_{icc}$ & 0.32 \\
			$\lambda_{cc}$ & 0.24\\
			Initial steps & 1000\\
			Batch size & 512\\
			Action repeat & 2\\
			Encoder learning rate & $10^{-3}$\\
			Encoder $\tau$ & $5\cdot10^{-3}$\\
                Decoder learning rate & $10^{-3}$\\
			Frames & 1000\\
			Actor learning rate & $10^{-3}$\\
			Critic learning rate & $10^{-3}$\\
                Critic $\tau$ & $10^{-2}$\\
			$\alpha$ learning rate & $10^{-4}$\\
			$\gamma$ & $0.99$\\
			Total Steps & $3.5\cdot10^{5}$\\
			Temperature & $0.1$\\
		\end{tabular}
	\end{center}
\end{table}
\end{document}